\documentclass{article}

\usepackage[nonatbib,preprint]{neurips_2022}
\usepackage[numbers]{natbib}

\usepackage[utf8]{inputenc} % allow utf-8 input
\usepackage[T1]{fontenc}    % use 8-bit T1 fonts
\usepackage{hyperref}       % hyperlinks
\usepackage{url}            % simple URL typesetting
\usepackage{booktabs}       % professional-quality tables
\usepackage{amsfonts}       % blackboard math symbols
\usepackage{nicefrac}       % compact symbols for 1/2, etc.
\usepackage{microtype}      % microtypography
\usepackage{xcolor}         % colors
\usepackage[pdftex]{graphicx}
\usepackage{amssymb}
\usepackage{mathtools}
\usepackage{amsthm}
\usepackage{cleveref}
\usepackage{wrapfig}
\usepackage{colortbl, color}
\usepackage{tikz}

\theoremstyle{plain}
\newtheorem{theorem}{Theorem}[section]

\newtheorem{lemma}[theorem]{Lemma}
\newtheorem{corollary}[theorem]{Corollary}
\theoremstyle{definition}
\newtheorem{definition}[theorem]{Definition}

\theoremstyle{remark}
\newtheorem{remark}[theorem]{Remark}

\newcommand{\R}{\mathbb{R}}
\newcommand{\er}{\mathsf{Res}}
\newcommand{\ind}{\mathbf{1}}

\newcommand{\resembed}[1]{\mathbf{r}_{#1}}
\newcommand{\approxrembed}[1]{\mathbf{\widehat{r}}_{#1}}

\renewcommand{\vec}{\mathbf}

\definecolor{Yellow}{rgb}{1.0,0.93,0}
\definecolor{Orange}{rgb}{1.0,0.68,0.25}
\definecolor{White}{rgb}{1.0,1.0,1.0}

\title{Affinity-Aware Graph Networks}

\author{%
  Ameya Velingker\thanks{Equal contribution}\\
  Google Research\\
  \texttt{ameyav@google.com} \\
  \And
  Ali Kemal Sinop\footnotemark[1]\\
  Google Research\\
  \texttt{asinop@google.com}
  \AND
  Ira Ktena\\
  DeepMind\\
  \texttt{iraktena@deepmind.com}\\
  \And
  Petar Veli\v{c}kovi\'c\\
  DeepMind\\
  \texttt{petarv@deepmind.com}\\
  \And
  Sreenivas Gollapudi\\
  Google Research\\
  \texttt{sgollapu@google.com}
}

\begin{document}

\maketitle

\begin{abstract}
Graph Neural Networks (GNNs) have emerged as a powerful technique for learning on relational data. Owing to the relatively limited number of message passing steps they perform—and hence a smaller receptive field—there has been significant interest in improving their expressivity by incorporating structural aspects of the underlying graph. In this paper, we explore the use of affinity measures as features in graph neural networks, in particular measures arising from random walks, including effective resistance, hitting and commute times. We propose message passing networks based on these features and evaluate their performance on a variety of node and graph property prediction tasks. 
Our architecture has lower computational complexity, while our features are invariant to the permutations of the underlying graph. The measures we compute allow the network to exploit the connectivity properties of the graph, thereby allowing us to outperform relevant benchmarks for a wide variety of tasks, often with significantly fewer message passing steps. 
%These tasks include previously proposed synthetic graph-algorithmic tasks, as well as various tasks from the Open Graph Benchmark (at various levels of scale). 
On one of the largest publicly available graph regression datasets, OGB-LSC-PCQM4Mv1, we obtain the best known single-model validation MAE at the time of writing.
\end{abstract}

\section{Introduction}
% \newtodo{Add some introductory material and references.}
Graph Neural Networks (GNNs) constitute a powerful tool for learning meaningful representations in non-Euclidean domains. GNN models have achieved significant successes in a wide variety of node prediction~\cite{hamilton2017inductive, luan2019break}, link prediction~\cite{zhang2018link, you2019position}, and graph prediction~\cite{duvenaud2015convolutional, ying2019gnnexplainer} tasks. These tasks naturally emerge in a wide range of applications, including autonomous driving~\cite{chen2019pct}, neuroimaging~\cite{parisot2018disease}, combinatorial optimization~\cite{gasse2019exact, nair2020solving}, and recommender systems~\cite{ying2018graph}, while they have enabled significant scientific advances in the fields of biomedicine~\cite{wang2021scgnn}, structural biology~\cite{jumper2021highly}, molecular chemistry~\cite{stokes2020deep} and physics~\cite{bapst2020unveiling}.

Despite the predictive power of GNNs, it is known that the expressive power of standard GNNs is limited by the 1-Weisfeiler-Lehman (1-WL) test~\cite{xu2018powerful}. Intuitively, GNNs possess the same power in terms of distinguishing between non-isomorphic (sub-)graphs, while having the added benefit of adapting to the given data distribution. For some architectures, two nodes with different local structures have the same computational graph, thus thwarting distinguishability in a standard GNN. Even though some attempts have been made to address this limitation with higher-order GNNs~\cite{morris2019weisfeiler}, most traditional GNN architectures fail to distinguish between such nodes.

\textbf{Contributions}: We propose the use of \emph{affinity metrics} as features in a graph neural network to circumvent this limitation. Specifically, we consider statistics that arise from random walks in graphs, such as \emph{hitting time} and \emph{commute time} between pairs of vertices. We present a means of incorporating these statistics as scalar edge features in a message passing neural network (MPNN)~\cite{gilmer2017neural}. Additionally, we present a set of vector-valued \emph{resistive embeddings} that can be incorporated as node or edge feature vectors in the network. We show that such embeddings can be efficiently approximated, even for larger graphs, using sketching and dimensionality reduction techniques.

Moreover, we evaluate our networks on a number of benchmark datasets of diverse scales. First, we show that our networks outperform other baselines on the PNA dataset~\cite{corso2020principal}, which includes 6 node and graph algorithmic tasks, showing the ability of affinity measures to exploit structural properties of graphs. We also evaluate the performance on a number of graph and node tasks for datasets in the Open Graph Benchmark (OGB) collection~\cite{hu2020open}, including molecular and citation graphs. In particular, our networks with scalar effective resistance edge features achieve the state of the art on the OGB-LSC PCQM4Mv1 dataset, which was featured in a KDD Cup 2021 competition for large scale graph representation learning.

\section{Related Work}
Our work builds upon a wealth of graph theoretical and graph representation learning works, while we focus on a supervised, inductive setting. %Various GNN architectures have been proposed that can be leveraged in an unsupervised or semi-supervised setting, but we do not capture those extensively to provide a complete picture of the most relevant research.

Even though GNN architectures were originally classified as spectral or spatial, we abstain from this division as recent research has demonstrated some equivalence of the graph convolution process regardless of the choice of convolution kernels~\citep{balcilar2021analyzing, bronstein2021geometric}. Spectrally-motivated methods require the eigendecomposition of the graph Laplacian matrix (or an approximation thereof) and, hence, corresponding convolutions capture different frequencies of the graph signal. Early works in this space include ChebNet~\cite{defferrard2016convolutional} and the more efficient kernel reparametrisation by Kipf et al.~\citep{kipf2017semi} to which we compare our work. Levie et al.~\citep{levie2018cayleynets} proposed CayleyNets, an alternative polynomial approximation of the Laplacian eigendecomposition.
%that is not computationally prohibitive for large graphs.

Message passing neural networks (MPNNs)~\cite{gilmer2017neural} perform a transformation of node and edge representations before and after an arbitrary aggregator (e.g. \textit{sum}). Graph attention networks (GATs)~\cite{velickovic2018graph} aimed to improve the expressivity of GNNs by allowing graph nodes to ``attend'' differently to different edges inspired by the success of transformers in NLP tasks. One of the most relevant works was proposed by Beaini et al.~\citep{beaini2021directional}, i.e. directional graph networks (DGN). DGN uses the gradients of the low-frequency eigenvectors of the graph Laplacian, which are known to capture key information about the global structure of the graph and prove that the aggregators they construct using these gradients lead to more discriminative models than standard GNNs according to the 1-WL test. Prior work~\cite{morris2019weisfeiler} used higher-order ($k$-dimensional) GNNs, based on $k$-WL, and a hierarchical variant and proved theoretically and experimentally the improved expressivity in comparison to other models.

Other notable works include Position-aware Graph Neural Networks~\cite{you2019position} that capture positions/locations of nodes
with respect to a set of anchor nodes, 
Distance Encoding Networks~\cite{LiWWL20} that use the first few powers of the normalized adjacency matrix
as node features,
and Graph Isomoprhism Networks (GINs)~\cite{xu2018powerful}. Hamilton et al.~\citep{hamilton2017inductive} proposed a method to constuct node representations by sampling a fixed-size neighborhood of each node, and then performing a specific aggregator over
it, which led to impressive performance on large-scale inductive benchmarks. Bouritsas et al.~\citep{bouritsas2020improving} use topologically-aware message passing to detect and count graph substructures, while Bornar et al.~\citep{bodnar2021weisfeiler} propose a message-passing procedure on cell complexes motivated by a novel colour refinement algorithm to test their isomorphism which prove to be powerful for molecular benchmarks. 

% [TODO:iraktena]
% Mention Benchmarking GNNs (mention large-scale GNNs that are based on sampling)
% works leveraging motifs attempt to capture graph structure
% Maybe refer to different benchmarking efforts to compare across methods in the graph representation learning domain.

\section{Affinity Measures and GNNs}
Our goal is to incorporate affinity measures related to random walk metrics on graphs 
into the GNN architecture.

\subsection{Random Walks, Hitting and Commute Times}
We define several natural properties of a graph that arise from a random walk. A random walk on $G$ starting from a node $u$ is a Markov chain on the vertex set $V$ such that the initial vertex is $u$, and at each time step, one moves from the current vertex to a neighbor, chosen with probability proportional to the weight of outgoing edges. We will use 
$\pi$ to denote the stationary distribution of this Markov Chain. For
random walks on weighted, undirected graphs, we know that $\pi_u = \frac{d_u}{2 M}$, where $d_u$ is the weighted
degree of node $u$, and $M$ is the sum of edge weights.

The \emph{hitting time} $H_{uv}$ fom $u$ to $v$ is defined as the expected number of steps for a random walk starting at $u$ to hit $v$.
We can also define the \emph{commute time} between $u$ and $v$ as $K_{uv}
= H_{uv} + H_{vu}$, the expected round-trip time for a random walk starting at $u$ to reach $v$ and then return to $u$.

\subsection{Effective Resistance}
A closely related quantity is the measure of \emph{effective resistances} in 
undirected graphs. This quantity corresponds to the effective resistance
if the whole graph was replaced with a circuit where each edge becomes
a resistor with resistance equal to the reciprocal of its weight. 
We will use $\er(u,v)$ to denote the effective resistance between
nodes $u$ and $v$. For undirected graphs, it is known that~\cite{lovasz93}
the effective resistance is proportional to the commute time, 
%\begin{equation}
\(\er(u,v) = \frac{1}{2 M} K_{uv}. 
\)
%\label{eq:effres-commute}
%\end{equation}
%

In light of the above, our broad goal is to incorporate effective resistances and hitting times as edge features in an MPNN, as we describe in Section~\ref{sec:featmpnn}.

\subsection{Resistive Embeddings}
Effective resistances allow us to define the \emph{resistive embedding}, a mapping that associates each node $v$ of a graph $G = (V,E,W)$, where $W$ are the non-negative edge weights, with an embedding vector. Before we specify the resistive embedding, we define a few terms. Let $L = D-A$ be the graph Laplacian of $G$, where $D \in \R^{n\times n}$ is the diagonal matrix containing the weighted degree of each node and $A\in \R^{n\times n}$ is the adjacency matrix, whose $(i,j)^{th}$ entry is equal to the edge weight between $i$ and $j$, if exists; and $0$ otherwise. Let $B$ be the $m\times n$ edge-node incidence matrix, where $|V| = n$ and $|E| = m$, defined as follows: The $i$-th row of $B$ corresponds to the $i$-th edge $e_i=(u_i, v_i)$ of $G$ and has a $+1$ in the $u_i$-th column and a $-1$ in the $v_i$-th column, while all other entries are zero. Finally we will use
$C \in \R^{m\times m}$ to denote the conductance matrix, which is a diagonal matrix with $C_{ii}$ being the weight of $i^{th}$ edge.
It is easy to verify that $B^T C B = L$. 
Even though $L$ is not invertible, its null-space consists of
the indicator vectors for every connected component of $G$. For example,
if $G$ is connected, then $L$'s nullspace consists only of the multiples of all-$1$'s vector
\footnote{Throughout this section,
we will assume that our graph is connected. However everything applies to disconnected graphs, too.}. Hence, for any vector $x$ orthogonal to all-$1$'s, $L \cdot L^{\dagger} x = x$, where $L^{\dagger}$ is the pseudo-inverse.

We can express effective resistance between
any pair of nodes using Laplacian matrices~\cite{lovasz93} as
\(
\er(u,v) = (\ind_u-\ind_v)^T L^{\dagger} (\ind_u-\ind_v), 
\) where $\ind_v$ is an $n$-dimensional vector specifying the indicator for node $v$.
We are now ready to define the resistive embedding.
\begin{definition}
\label{eq:resembed} (Effective Resistance Embedding)
 %Given a graph $G = (V,E)$ with $n$ nodes and %$m$ edges,
 %we define the 
 %\emph{resistive embedding} of $G$
 %is defined as %$\resembed{}: V \to \R^m$ as 
 %\begin{equation}
   \(\resembed{v} = C^{1/2} B L^{\dagger} \ind_v. \)
   %\label{eq:resembed}
 %\end{equation} 
\end{definition}
The useful property for us is that the effective resistance between two nodes in the graph can be obtained easily from the distance
between their corresponding embedding. 
\begin{lemma}\label{lem:embeddist}
 For any pair of nodes $u,v$, we have $\|\resembed{u} - \resembed{v}\|_2^2 = \er(u,v)$.
\end{lemma}
The proof of this lemma can be found in \Cref{apx:proofs}.
%
% \begin{proof}
% \begin{align*}
%     \|\resembed{u} - \resembed{v}\|_2^2 &= \|C^{1/2} B L_G^{-1} (\ind_u-\ind_v)\|_2^2 \\
%     &= (\ind_u-\ind_v)^T L^{\dagger} (B^T C B) L^{\dagger} (\ind_u-\ind_v) \\
%     &= (\ind_u-\ind_v)^T L^{\dagger} L L^{\dagger} (\ind_u-\ind_v) \\
%     &= (\ind_u-\ind_v)^T L^{\dagger} (\ind_u-\ind_v) 
%     = \er(u,v). \tag*{\qedhere}
% \end{align*}
%\todo{The last line above is unclear because effective resistance was 
%defined i a different way. We need to state this fact in the previous 
%section. Ali: DONE}
%\end{proof}
%
%\newtodo{Can someone check this blurb?}
One can easily check that any rotation of $\resembed{}$ also satisfies 
\Cref{lem:embeddist}, since rotations preserve Euclidean distances. 
In particular, if $U$ is an orthonormal matrix, then
$U \resembed{}$ is also a valid resistive embedding. 
In other words, 
resistive embeddings are unique with respect to rotations. 
This poses a challenge if we want to use the resistive embeddings as node or edge features: %Even on the same graph, if one chose a different basis
%to compute the resistive embedding, the output of the network might be different. 
%In other words, 
we need a way to enforce that a (G)NN using them will do so in a way that is invariant or equivariant to any rotations of the embeddings.
In our current work, we rely on data augmentation: at every training iteration, we apply random rotations to the input ER embeddings.
\begin{remark}
While data augmentation is a popular approach for promoting invariant and equivariant predictions, it is only \emph{hinting} to the network that such predictions are favourable. It is also possible, in the spirit of the geometric deep learning blueprint \citep{bronstein2021geometric}, to combine ER embeddings with an $O(n)$-equivariant GNN, which rigorously enforces rotational equivariance. A popular approach to building equivariant GNNs has been proposed by ~\citep{satorras2021n}, though it focuses on the full Euclidean group $E(n)$ rather than $O(n)$. We leave this exploration to future work.
%the resistive embedding is not translation invariant -- 
%the sum of $\resembed{u}$'s should be $0$).
%We leave this as future work.
\end{remark}

\begin{definition}
\label{def:mean-er}
 Let $\mathbf{p} := \sum_u \pi_u \resembed{u}$ be the mean of effective resistance embedding. %w.r.t. a stationary distribution.
\end{definition}
We might view $\mathbf{p}$ as a ``weighted mean''\footnote{Note that the average of
all $\resembed{u}$'s will be $0$. If the graph is regular, then $\mathbf{p}$ will also be $0$.} of $\resembed{}$.
We will define
the hitting time radius, $H_{\max}$, of a given graph as the 
maximum hitting time between any two nodes:
\begin{definition}[Hitting Time Radius]
\label{def:eff-res-rad}
$H_{\max} := \max_{u,v} H_{u,v}$.
\end{definition}
We will need the following to bound the hitting times
we computed:
\begin{lemma}
\label{thm:dist-to-kmax}
For any node $u$, 
$\|\resembed{u} - \mathbf{p}\|^2 \le \frac{H_{\max}}{M}$.
\end{lemma}
The proof of \Cref{thm:dist-to-kmax} follows immediately
from the fact that $\mathbf{p}$ is a convex combination
of all $\resembed{}$'s and Jensen's inequality.
\subsection{Incorporating Features into MPNNs} \label{sec:featmpnn}
We reiterate that our main aim is to demonstrate (theoretically and empirically) that there are good reasons to incorporate affinity-based measures into GNN computations.

In the simplest instance, a method that improves a GNN's expressive power may compute additional \emph{features} (positional or structural) which would assist the GNN in discriminating between examples it otherwise wouldn't (easily) be able to. These features are then appended to the GNN's inputs for further processing. For example, it has been shown that endowing nodes with a one-hot based \emph{identity} is already sufficient for improving expressive power \citep{murphy2019relational}; this was then relaxed to any randomly-sampled scalar feature by Sato et al.~\citep{sato2021random}. It is, of course, possible to create dedicated features that even count substructures of interest \citep{bouritsas2020improving}. Further, the adjacency information can be factorised \citep{qiu2018network} or eigendecomposed \citep{dwivedi2021graph} to provide useful structural embeddings for the GNN.

\begin{wrapfigure}{r}{0.3\textwidth}
    \centering
    \includegraphics[width=0.25\textwidth]{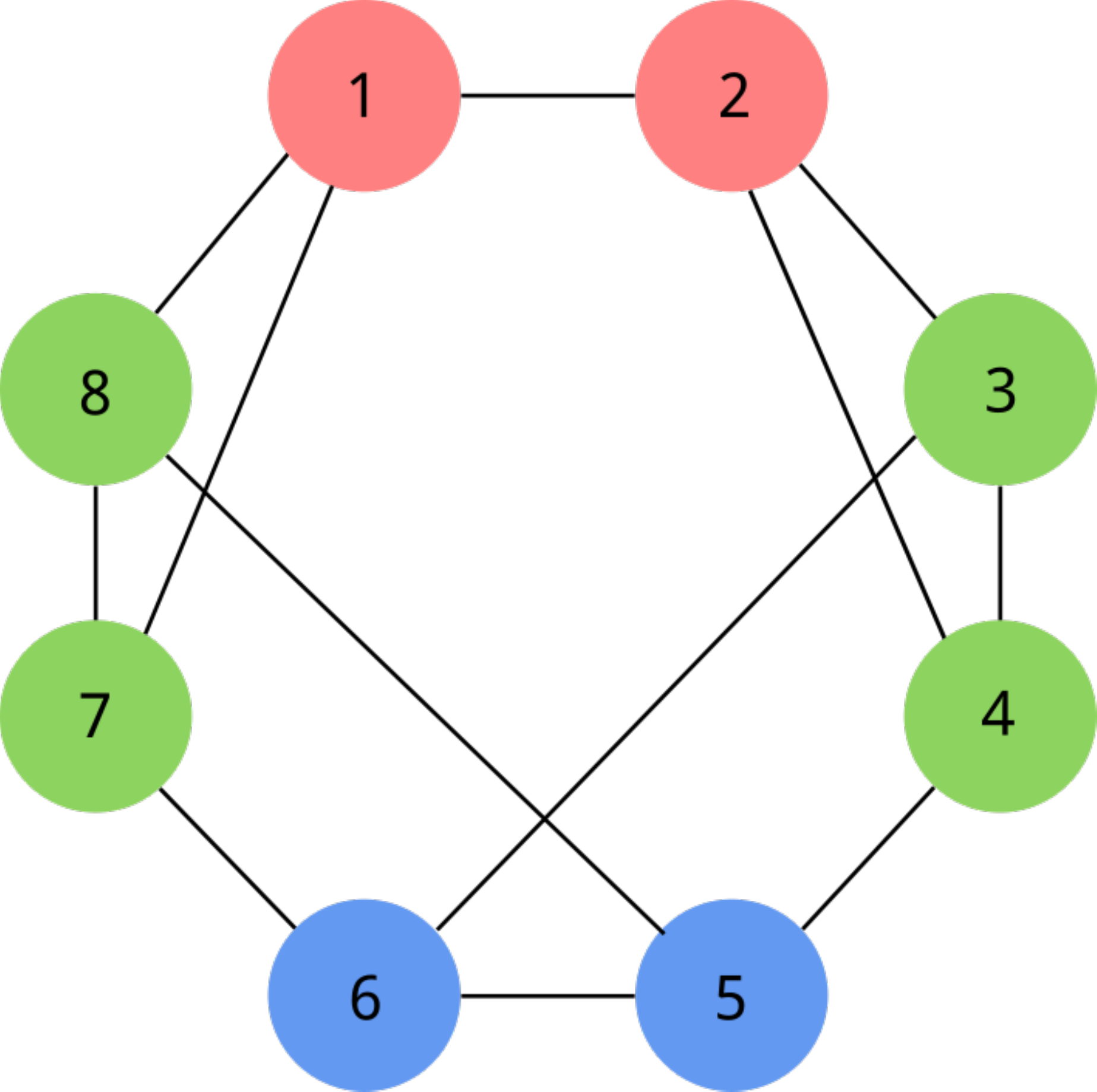}
    \caption{Degree 3 graph on 8 nodes, with isomorphism classes indicated by colors. While nodes of the same color are structurally identical, nodes of different colors are not. A standard GNN limited by the 1-WL cannot distinguish between nodes of different colors. However, affinity based networks that use effective resistances, hitting times, or resistive embeddings can distinguish every pair of such nodes.}
    \vspace{-30pt}
    \label{fig:deg3graph}
\end{wrapfigure}

We will focus our attention on exactly this class of methods, as it is a lightweight and direct way of demonstrating improvements from these computations. Hence, our baselines will all be instances of the MPNN framework \citep{gilmer2017neural}, which we will attempt to improve by endowing them with affinity-based features. We start by theoretically proving that these features indeed improve expressive power:

\begin{theorem}
MPNNs that make use of any one of (a) effective resistances, (b) hitting times, (c) resistive embeddings are strictly more powerful than the WL-1 test.
\end{theorem}

\begin{proof}
Since the networks in question arise from augmenting standard MPNNs with additional node/edge features, we have that these networks are at least as powerful as the 1-WL test.

In order to show that these networks are strictly \emph{more powerful} than the 1-WL test, it suffices to show the existence of a graph for which our affinity measure based networks can distinguish between certain nodes that a standard GNN (limited by the 1-WL test) cannot.

We present an example of a 3-regular graph (see Figure~\ref{fig:deg3graph}) on 8 nodes.
%(see Figure \ref{fig:deg3graph}).
%
It is well-known that a standard GNN that is limited by the 1-WL test cannot distinguish any pair of nodes in a regular graph, as the computation tree rooted at any node in the graph looks identical. However, there are  three isomorphism classes of nodes in the above graph (denoted by different colors), namely, $V_1 = \{1,2\}$, $V_2 = \{3,4,7,8\}$, and $V_3 = \{5,6\}$.

We now show that GNNs with affinity based measures can distinguish between a node in $V_i$ and a node in $V_j$, for $i\neq j$. We note that the hitting time from $a$ to $b$ depends only on the isomorphism classes of $a$ and $b$. Thus, we write $r_{i,j}$ as the effective resistance between a node in $V_i$ and a node in $V_j$. Note that $r_{i,j} = r_{j,i}$, and it is easy to verify that
%\begin{multline*}
\(
 r_{1,1} = 2/3, r_{2,2} = 15/28, r_{3,3} = 4/7, 
 r_{1,2} = r_{2,1} = 185/336, r_{2,3} = r_{3,2} = 209/336.
\)
%\end{multline*}
Hence, it follows that in a message passing step of an MPNN that uses effective resistances, vertices in $V_1$, $V_2$, and $V_3$ will aggregate feature multisets $\{r_{1,1}, r_{1,2}, r_{1,2}\} = \{2/3, 185/336, 185/336\}$, $\{r_{2,1}, r_{2,2}, r_{2,3}\} = \{185/336, 15/28, 209/336\}$, and $\{r_{2,3}, r_{2,3}, r_{3,3}\} = \{209/336, 209/336, 4/7\}$, respectively, all of which are all distinct multisets. Hence, such an MPNN can distinguish nodes in $V_i$ and $V_j$, $i\neq j$ for a suitable aggregation function.

Note that if instead of scalar effective resistance features, an MPNN uses hitting time features or resistive embeddings, it can distinguish the same isomorphism classes as above because the effective resistance between nodes is a function of the two hitting times in either direction as well as of the resistive embeddings of the two nodes (see \cref{lem:embeddist}). In other words, one can reduce to the case of scalar effective resistance features by composing the aggregation function with the appropriate function that transforms either hitting times or resistive embeddings to effective resistances. {\qedhere}

%Thus, we write the normalized hitting time $h_{i,j}$ to denote $H(a,b)/2m = H(a,b)/24$ for $a\in V_i$, $b\in V_j$. It is easy to verify that:
%\begin{multline*}
% h_{1,1} = 1/3, h_{2,2} = 15/56, h_{3,3} = 2/7, h_{1,2} = 47/168,\\ h_{2,1} = 13/48, h_{2,3} = 33/112, h_{3,2} = 55/168
%\end{multline*}

%Moreover, routine computation reveals that the normalized  hitting times between nodes of different classes are
\end{proof}
% \newtodo{Petar: I am assuming that another proof will follow here. If not, remove this paragraph.}
% We can take our theoretical argument even further, and prove that affinity-based features offer expressive power benefits even over state-of-the-art higher-order GNNs:
% %\begin{theorem}
% %There exists a collection of graphs for which 
% %Collection of graphs where nodes cannot be distinguished by %DE-GNN, Position Aware GNNs but can be using ER features.
% %\end{theorem}
% %\begin{proof}
% %
% %\end{proof}

% \newtodo{Petar: This section after the last proof.}
We also compare our method to other strong feature-based approaches in the literature. The recently proposed DE-GNNs~\cite{LiWWL20} are arguably one of the closest proposals to ours, as they compute distance-encoded features. These features can be at least as powerful as our proposed affinity-based features \emph{if} polynomially many powers of the adjacency matrix are used. However, for all but the smallest graphs, using this many powers will be impractical---in fact, \citep{LiWWL20} only use powers of $A$ up to 3, which would not be able to reliably approximate affinity-based features. 

Ultimately, a technique that improves a GNN's expressive power can do so in \emph{three} broad directions. While we focus on the feature-based direction in this paper, we also acknowledge that it in no way compels the GNN to use the additional provided features. Hence, we briefly survey the other two, as an indication of the future research in affinity-based GNNs we hope this work will inspire.

One avenue involves modulating the \emph{message passing rule} to make advantage of the desired computations. Popular recent examples of this include DGN \citep{beaini2021directional} and LSPE \citep{dwivedi2021graph}. DGNs leverage the graph's Laplacian eigenvectors, but they do not merely use them as input features; instead, they define a directional \emph{vector field} based on the eigenvectors, and use it explicitly to anisotropically aggregate neighbourhoods. LSPE features a ``bespoke pipeline'' for processing positional inputs.

The other direction is to modulate the \emph{graph} over which messages are passed, usually by adding new nodes that correspond to desired substructures. An early proponent of this is the work of  \citep{morris2019weisfeiler}, which explicitly performs message passing over $k$-tuples of nodes at once. Recently, scalable efforts in this direction focus on carefully chosen substructures, e.g., junction trees~\citep{fey2020hierarchical}, cellular complexes \citep{bodnar2021weisfeiler}.

\subsection{Effective Resistance vs. Shortest Path Distance}
It is interesting to compare the use of effective resistances with shortest path distances (SPDs) in GNNs, given the considerable number of recent works that make use of SPDs (e.g., Graphormer~\citep{YingCLZKHSL21}, Position-Aware GNN~\citep{you2019position}, DE-GNN~\cite{LiWWL20}). The most direct comparison of our effective resistance-based MPNNs would be to use SPDs as edge features in the MPNNs. However, note that SPDs along graph edges are a trivial feature (unlike effective resistances, which still incorporate useful information about the global graph structure).

An alternative to edge features would be to use (a) SPDs to a small set of \emph{anchor nodes} as features in an MPNN (e.g., P-GNN \cite{you2019position}) or (b) a dense featurization incorporating shortest paths between all pairs of nodes (e.g., the dense attention mechanism in Graphormer~\cite{YingCLZKHSL21}). We remark that the latter approach typically incurs an $O(n^2)$ overhead, which our MPNN-based approach avoids.

We empirically compare our models to MPNNs that use approaches (a) and (b). Results on the PNA dataset show that our effective resistance-based GNNs outperform these approaches. Furthermore, we complement these empirical results with a theoretical result showing that under a limited number of message-passing steps, effective resistance features can allow one to distinguish structures that cannot be done using shortest path features. We point the reader to the appendix for these results.

%For such features, our experiments~\Cref{tab:pna} show that ER features yield better performance. Finally, one can also add the shortest path distances between distant pairs, which results in a dense network, as done in Graphormer. Note that the quadratic dependence severely restricts the classes of graphs that can be used, whereas our method is more scalable. Again, our experiments show that ER features perform better than Graphormer (\Cref{tab:pcqm4m}).

\section{Efficient Computation of Affinity Measures}
In order to use our features, it is important that they be computable efficiently. In this section, we show how to compute or approximate the various random walk-based affinity measures.

\subsection{Reducing Dimensionality of Resistive Embeddings}
For each edge, we seek to use the resistive embeddings as GNN features instead of the effective resistance. Now, the difficulty with using resistive embeddings directly as GNN features is the fact that the embeddings have dimension $m$, which can be quite large, e.g., up to $n^2$ for dense graphs.
It was shown by
\citep{SpielmanS11} that one can
reduce the dimensionality of the embedding while approximately preserving Euclidean distances. 
The idea is to use a random projection via a constructive version of the Johnson-Lindenstrauss Lemma:
\begin{lemma}[Constructive Johnson-Lindenstrauss]\label{lem:jl}
 Let $x_1, x_2, \ldots, x_n \in \R^d$ be a set of $n$ points; 
 and let $\alpha_1, \alpha_2, \ldots, \alpha_m \in \R^n$ be 
 fixed linear combinations. 
 Suppose $\Pi$ is a $k\times d$ matrix whose entries are chosen i.i.d. from a Gaussian $N(0,1)$
 and consider $\widehat{x_i} := \frac{1}{\sqrt{k}} \Pi x_i$.
 
 Then, it follows that for $k \geq C \log(m n)/\epsilon^2$, with probability $1-o(1)$, we have
 %\begin{align*}
 \(
   \|\sum_j \alpha_{i,j} \widehat{x}_j\|^2
   = (1\pm \epsilon) \|\sum_j \alpha_{i,j} x_j\|_2^2
\)   
 %\end{align*}
 %
for every $1\leq j \leq m$.
\end{lemma}
Using triangle inequality, we can see that the inner products
between fixed linear combinations are preserved up to some additive error: 
\begin{corollary}
\label{thm:jl-inner-product}
For any fixed vectors
$\alpha,\beta \in \R^n$,
if we let $X := \sum_i \alpha_i {x}_i$, $\widehat{X} := \sum_i \alpha_i \widehat{x}_i$ and similarly 
$Y := \sum_i \beta_i {x}_i$, 
$\widehat{Y} := \sum_i \beta_i \widehat{x}_i$; then 
\(
\left|
\langle X, Y \rangle 
- \langle \widehat{X}, \widehat{Y} \rangle \right|
\le \frac{\epsilon}{2} \left(\|X\|^2 + \|Y\|^2\right).
\)
\end{corollary}
The proof of \Cref{thm:jl-inner-product} can be found in
\Cref{apx:proofs}.

Therefore, we can choose a desired $\epsilon > 0$ and $k = O(\log(n)/\epsilon^2)$ and instead use $\approxrembed{}: V \to \R^k$ as the embedding, where
\(
%\begin{equation}
 \approxrembed{v} = \frac{1}{\sqrt{k}} \Pi B L^{\dagger} e_v %\label{eq:approxembed}
\)
%\end{equation}
for a randomly chosen $k\times d$ matrix $\Pi$ whose entries are i.i.d. Gaussians from $\mathcal{N}(0,1)$. Then, by \Cref{lem:embeddist} and \Cref{lem:jl}, we have that for every edge $(u,v)\in E$,
\(
 \|\approxrembed{u} - \approxrembed{v}\|_2^2 = (1 \pm 3\epsilon) \er(u,v)
\)
with probability at least $1 - \frac{1}{n^2}$.

So the computation of random embeddings, $\approxrembed{}$, requires
solving $O((n+m) \log n / \epsilon^2)$ many Laplacian linear systems.
By using one of the fast Laplacian solvers \cite{KoutisMP14}, 
we can compute the random embeddings in the near-linear time.
Hence the total running time becomes $\widetilde{O} (n+m)$.
%\left((n+m) \log^{3/2} n \mathrm{poly} \log \log n / \epsilon^2 \right)$.

\subsection{Fast Computation of Hitting Times}
Note that it is not clear how to compute hitting times efficiently as in the case of commute times / effective resistances. 
The naive approach involves solving a linear system for each edge, 
resulting in a running time of $\Omega(n m)$, which is
prohibitive. 
One of our technical contributions in this paper is a method for fast computation of hitting times. In particular, we will show how to use the
approximate effective resistance embeddings, $\approxrembed{}$, to 
obtain an estimate for hitting times with additive error. 

%
% Recall that, for an undirected graph, the stationary distribution
% of a random walk, $\pi$, is given by $\pi_u = \frac{d_u}{2 M}$, where $d_u$ is the weighted degree of $u$. 
Let $\widehat{\mathbf{p}} := \sum_u \pi_u \approxrembed{u}$. 
Just like $\approxrembed{}$ being an approximation of $\resembed{}$,
$\widehat{\mathbf{p}}$ is an approximation of $\mathbf{p}$.
Consider the following quantity, 
%\begin{equation}
    %\label{eq:hitting-time-fast}
    \(
    \widehat{H}_{u,v}
    = 2 M 
    \langle \approxrembed{v}-\approxrembed{u}, \approxrembed{v} - \widehat{\mathbf{p}}\rangle\).
%\end{equation}
%
We will use this quantity as an approximation of $H_{u,v}$.
In the following part, we will bound the difference
between $H_{u,v}$ and $\widehat{H}_{u,v}$. Our starting
point will be expressing $H_{u,v}$ in terms of
the effective resistance embeddings. 
\begin{lemma}
\label{thm:hit-to-resembed}
${H}_{u,v} = 2 M 
\langle \resembed{v}-\resembed{u}, \resembed{v} - {\mathbf{p}}\rangle$
where $\mathbf{p} := \sum_u \pi_u \resembed{u}$. 
\end{lemma}
The proof of \Cref{thm:hit-to-resembed} can
be found in \Cref{apx:proofs}.

\begin{theorem}
\label{thm:approx-ht}
$|\widehat{H}_{u,v} - H_{u,v}| \le  
3 \epsilon H_{\max}$.
% where $H_{\max}$ is defined as in \Cref{def:eff-res-rad}.
\end{theorem}
\begin{proof}
Using~\Cref{thm:hit-to-resembed}, we see that
$|\widehat{H}_{u,v} - H_{u,v}|$ can be bounded as
\begin{align*}
    2 M \left| \langle \approxrembed{v}-\approxrembed{u}, \approxrembed{v} - \widehat{\mathbf{p}}\rangle
    - \langle \resembed{v} - \resembed{u}, \resembed{v} - \mathbf{p}
    \rangle \right| 
    \le \epsilon M \left(
    \|\resembed{v} - \resembed{u}\|^2 
    + \|\resembed{v} - \mathbf{p}\|^2\right) 
    \le 3 \epsilon H_{\max}, 
\end{align*} 
where we used \Cref{thm:jl-inner-product} in the first inequality and
\Cref{def:eff-res-rad} in the last inequality.
% Since $\pi_i$'s form a distribution, using Jensen's inequality, we see
% that:
% \[
% \|\resembed{v} - \mathbf{p}\|^2
% \le \sum_i \pi_i \|\resembed{v}-\resembed{i}\|^2
% = \sum_i \pi_i \frac{K_{u,i}}{2 M} = \frac{K_u}{2 M}.
% \tag*{\qedhere}
% \]
\end{proof}

\section{Experiments}

\subsection{Baselines}

As previously discussed, our empirical evaluation seeks to show benefits from endowing standard expressive GNNs with additional affinity-based features. All architectures we experiment with will therefore conform to the message passing neural network (MPNN) blueprint \citep{gilmer2017neural}, which we now briefly describe for convenience.  

Assume that our input graph, $\mathcal{G}=(\mathcal{V},\mathcal{E})$, has node features $\vec{x}_u\in\mathbb{R}^n$, edge features $\vec{x}_{uv}\in\mathbb{R}^m$ and graph-level features $\vec{x}_\mathcal{G}\in\mathbb{R}^l$, for nodes $u, v\in\mathcal{V}$ and edges $(u, v)\in\mathcal{E}$. We provide encoders $f_n:\mathbb{R}^n\rightarrow\mathbb{R}^k$, $f_e:\mathbb{R}^m\rightarrow\mathbb{R}^k$ and $f_g:\mathbb{R}^l\rightarrow\mathbb{R}^k$ that transform these inputs into a latent space:
\begin{equation}\label{eqn:enc}
    \vec{h}_u^{(0)} = f_n(\vec{x}_u) \qquad \vec{h}_{uv}^{(0)} = f_e(\vec{x}_{uv}) \qquad
    \vec{h}_\mathcal{G}^{(0)} = f_g(\vec{x}_\mathcal{G})
\end{equation}
Our \textit{MPNN} then performs several message passing steps 
\(
%\begin{equation}
%\label{eqn:proc}
    \vec{H}^{(t+1)} = P_{t+1}(\vec{H}^{(t)}) 
%\end{equation}
\)
where $\vec{H}^{(t)} = \Big(\big\{\vec{h}_u^{(t)}\big\}_{u\in\mathcal{V}}, \big\{\vec{h}_{uv}^{(t)}\big\}_{(u,v)\in\mathcal{E}}, \vec{h}_\mathcal{G}^{(t)}\Big)$ contains all of the latents at a particular processing step $t\geq 0$. 

This process is iterated for $T$ steps, recovering final latents $\vec{H}^{(T)}$. These can then be \emph{decoded} into node-, edge-, and graph-level predictions (as required), using analogous decoder functions $g_n$, $g_e$, $g_g$:
\begin{equation}
    \vec{y}_u = g_n(\vec{h}_u^{(T)}), \qquad \vec{y}_{uv} = g_e(\vec{h}_{uv}^{(T)}), \qquad
    \vec{y}_\mathcal{G} = g_g(\vec{h}_\mathcal{G}^{(T)})
\end{equation}
Generally, $f$, $g$ are MLPs, while we use an MPNN update rule for $P$. It computes message vectors, $\vec{m}^{(t)}_{uv}$, to be sent across the edge $(u, v)$ and then aggregates them at the receiver nodes as follows:
\begin{align}\label{eqn:mpnn1}
    \vec{m}^{(t+1)}_{uv} = \psi_{t+1}\big(\vec{h}^{(t)}_u, \vec{h}^{(t)}_v, \vec{h}^{(0)}_{uv}\big);
    \quad
    %\label{eqn:mpnn2}
    \vec{h}^{(t+1)}_u = \phi_{t+1}\Big(\vec{h}^{(t)}_u, \sum_{u\in\mathcal{N}_v} \vec{m}^{(t+1)}_{vu}\Big).
\end{align}
The message function $\psi_{t+1}$ and the update function $\phi_{t+1}$ are both MLPs. All of our models have been implemented using the Jraph library \citep{jraph2020github}.

Occasionally, the dataset in question will be easy to overfit with the most general form of message function (\cref{eqn:mpnn1}). In these cases, we resort to assuming that $\psi$ factorises into an \emph{attention mechanism}
%\begin{equation} \label{eqn:attention1}
\(
    \vec{m}^{(t+1)}_{uv} = a_{t+1}\big(\vec{h}^{(t)}_u, \vec{h}^{(t)}_v, \vec{h}^{(0)}_{uv}\big)\psi_{t+1}\big(\vec{h}^{(t)}_u\big)
\), 
%\end{equation}
where the attention function $a$ is scalar-valued. We will refer to this particular MPNN baseline as a graph attention network (\textit{GAT}) \citep{velickovic2018graph}.

When relevant, we may also recall the results that a particular strong baseline (such as DGN \citep{beaini2021directional} or Graphormer \citep{Ying2021DoTR}) achieves on a dataset of interest. Note that these baselines modulate the message passing procedure rather than appending features, and are, hence, a different category to our method---their performance is provided for indicative reasons only. Where appropriate, we will use ``\textit{DGN (features)}'' to refer to an MPNN that uses the eigenvector flows as additional edge features, without modulating the mechanism.

% NOTE: the values in this table correspond to the best average score for a single set of hyperparameters across all tasks
\begin{table*}[hbtp]
    \centering
    \scalebox{0.95}{
    \begin{tabular}{cccccccc}
         \toprule
          \multicolumn{2}{c}{} & \multicolumn{3}{c}{\bf Node tasks} & \multicolumn{3}{c}{\bf Graph tasks} \\
         {\bf Model}                &  {\bf Avg score} & SSSP & Ecc & Lap feat & Conn & Diam & Spec rad \\ 
         \midrule %\hline
         GAT            	  &  -1.730 &  -2.213 & -1.935 &	 	-2.644 & -0.618 & -1.430 &	-1.538	\\ %\hline
         GCN                  & -1.592 & {-2.283}  &	-1.978 &	 	-1.698 & -0.618 & -1.432 & -1.541  \\% \hline
         MPNN                 & -2.665 & -2.235 & -2.419 & -3.116 & -1.887 & -2.681 & -3.652 \\ %\hline
         MPNN (rand features) & -2.490 & -2.136 &	-1.808 &	-3.873 & -1.696 & -2.614 & -2.813  \\ %\hline
         DGN (features)                  & -2.743 & -2.165 & -1.911 &	-4.184 & -1.858 &	-2.814 & -3.528 \\ \midrule %\hline
         \textbf{ER GNN}     & -2.779 & -2.146 & -1.869 & -3.945 & \textbf{-1.962} & \textbf{-2.940} & {-3.811} \\ %\hline
         \textbf{ER (node) embeddings} & -2.658 & -2.245 &	{-2.493} & -3.533 & -1.649 &	-2.886 & -3.144 \\ %\hline
         \textbf{ER (edge) embeddings} & {-2.789} & -2.266 &	-2.125 & -4.253 & -1.664 & -2.807 & -3.617  \\ %\hline
          \textbf{Hitting Times} 
          & 
          {-2.816} & -2.189 & -1.904 & \textbf{-4.397} & -1.888 & -2.796 & -3.720  \\ 
          \midrule
          \textbf{All ER features} &
          \textbf{-3.106} &
          \textbf{-2.789} &  \textbf{-3.082} & -4.047 & -1.858 & -2.894 & \textbf{-3.962} \\
        %   \textbf{HT + ER node embeddings (rand rots)} & -2.922 & -2.391 & -2.805 & -3.584 & -1.952 & -3.040 & -3.761 \\
          \bottomrule
    \end{tabular}
    }
    \caption{$log(MSE)$ on the PNA test dataset}
    \label{tab:pna_table}
\end{table*}

\subsection{Datasets}
\subsubsection{PNA dataset}
The first dataset we explore is the PNA dataset~\cite{corso2020principal}, which captures a multimodal setting. This consists of a collection of \textit{node tasks}, i.e., \textbf{(1)} Single-source shortest paths, \textbf{(2)} Eccentricity and \textbf{(3)} Laplacian features, as well as \textit{graph tasks}, i.e. \textbf{(4)} Connectivity, \textbf{(5)} Diameter and \textbf{(6)} Spectral radius. PNA is a set of structured tasks that complements our other datasets. Our results are given
in \Cref{tab:pna_table}.

As we can see from the table, even adding a single feature, 
effective resistance (ER GNN), yields the best average
score compared to other models. As expected using 
hitting times as edge features improve upon effective resistances. However once we combine
all ER features, which include effective resistances, hitting times as well as node and edge embeddings, we get 
the best scores. On these structured tasks, we can
see that the affinity based measures provide a significant
advantage.

\begin{wraptable}{r}{0.54\textwidth}
    \centering
    \footnotesize
    \begin{tabular}{p{40mm}p{24mm}}
      \toprule
       & \textbf{MolHIV} \newline Test \% ROC-AUC  \\
       \midrule 
       % GAT &  \\ \hline
       GCN & $76.06 \pm 0.97$ \\ 
       MPNN & $74.67 \pm 0.19$ \\ 
       MPNN + Random Features & $75.52 \pm 1.07$ \\
       DGN & $\textbf{79.70} \pm 0.97$ \\ \midrule
       \textbf{ER GNN} & $77.75 \pm 0.426$ \\ 
       \textbf{Hitting Times} & $76.56\pm 0.915$ \\ 
       \textbf{ER (node) embeddings} & $77.77 \pm 0.339$ \\ 
       \textbf{ER (edge) embeddings} & $76.18 \pm 0.992$ \\ 
       \textbf{ER (node) embeddings + HT} & $78.16 \pm 0.792$  \\ 
       \textbf{ER (node) embeddings} 
       \newline 
       (with random rotations) & $76.28\pm 0.541$ \\ 
       %\textbf{ER (edge) embeddings} \newline (with random rotations) & 74.98 \\ 
       \textbf{ER (node) embeddings + HT} \newline 
       (with random rotations) & \textbf{$\textit{79.13} \pm 0.358$}  \\ 
       \bottomrule
    \end{tabular}
    \caption{Test \% AUC-ROC averaged over 5 seeds.}
    \vspace{-2mm}
    \label{molhiv_table}
\end{wraptable}

\subsubsection{Small molecule classification: ogbg-molhiv}
The \textit{ogbg-molhiv} dataset is a molecular property 
prediction dataset comprised of molecular graphs
without spatial information (such as atom coordinates). 
Each graph corresponds to a molecule, with nodes representing atoms
and edges representing chemical bonds. Each node has an associated
$9$-dimensional feature, containing atomic number and chirality, as well as other additional atom features such as formal charge and whether the atom is in the ring or not. The goal is to predict whether a molecule inhibits HIV virus replication or not (see \Cref{molhiv_table}). 

On this dataset, effective resistances provide an improvement
over the standard MPNN. We achieve the best performance
using ER node embeddings and hitting times with random rotations. With these features, our network achieves
${\bf 79.13\%} \pm 0.358$ test accuracy, which is close to DGN.

\subsubsection{Multi-task molecular classification: ogbg-molpcba}

The \textit{ogbg-molpcba} dataset comprises molecular graphs without spatial information (such as atom coordinates). The aim is to classify them across 128 different biological activities (a single molecule may exhibit multiple activities). We follow the baseline MPNN architecture and evaluation from \citep{godwin2022simple}, including its use of the recently proposed simple and effective regulariser, Noisy Nodes. Additionally, we experiment with the use of additional random features.

\begin{table*}[bp]
    \centering
    \scalebox{0.95}{
    \begin{tabular}{cccc}
    \toprule
        \multicolumn{1}{c}{} & \multicolumn{3}{c}{\bf Test Mean Average Precision}  \\
         {\bf Model} & {\em 4 layers} & {\em 8 layers} & {\em 16 layers}  \\
     \midrule
         MPNN \citep{godwin2022simple} & 27.75\% $\pm$ 0.20  & 27.91\% $\pm$ 0.22 & 27.64\% $\pm$ 0.25\\
         MPNN + Noisy Nodes \cite{godwin2022simple} & 27.92\% $\pm$ 0.11 & 28.07\% $\pm$ 0.14 & {\bf 28.29}\% $\pm$ 0.13 \\ 
         MPNN + Random Features & 26.69\% $\pm$ 0.18 & 26.94\% $\pm$ 0.23 & 27.06\% $\pm$ 0.21 \\
         MPNN + Random Features + Noisy Nodes & 27.25\% $\pm$ 0.13 & 27.55\% $\pm$ 0.17 & 27.90\% $\pm$ 0.18 \\        
         \midrule
         MPNN + Noisy Nodes + ER (ours) & {\bf 28.11}\% $\pm$ 0.19 & 28.27\% $\pm$ 0.17 & 28.28\% $\pm$ 0.14 \\
         MPNN + Noisy Nodes + HT (ours) & 28.03\% $\pm$ 0.15 & \underline{\bf 28.32}\% $\pm$ 0.13 & 28.20\% $\pm$ 0.19  \\
     \bottomrule
    \end{tabular}
    }
    \caption{ogbg-molpcba performance for various model depths. Best performance across all models is underlined.}
    \label{tab:molpcba}
\end{table*}

Mirroring the evaluation protocol of \citep{godwin2022simple}, Table \ref{tab:molpcba} compares the performance of incorporating ER and hitting time (HT) features into the baseline MPNN models with Noisy Nodes, at various depths. What can be noticed is that models utilising affinity-based features are capable of reaching as well as exceeding peak test performance (in terms of mean average precision). However, what's important is the effect of these features at lower depths: it is possible to achieve comparable or better levels of performance with \underline{half the layers}, when utilising ER or HT features. This result illustrates the potential benefit affinity-based computations can have on molecular benchmarks, especially when no spatial geometry is provided as input.

\subsubsection{Scaling to larger graphs: ogbn-arXiv}

Most expressive GNNs that rely on the computation of structural features have not been scaled beyond small molecular datasets (such as the ones discussed in prior sections). This is due to the fact that computing them requires (time or storage) complexity which is at least quadratic in the graph size---making them inapplicable even for modest-sized graphs. This is, however, not the case for our proposed affinity-based metrics. We demonstrate this by scalably computing them on a larger-scale node classification benchmark, \textit{ogbn-arXiv} (a citation network with the goal of predicting the arXiv category of each paper). \textit{ogbn-arXiv} has 169,343 nodes and 1,166,243 edges, making quadratic approaches infeasible.

As MPNN models overfit this transductive dataset quite easily, the dominant approach to tackling it are graph attention networks (GATs) \citep{velickovic2018graph}. Accordingly, we trained a simple four-layer GAT on this dataset, achieving 72.02\% $\pm$ 0.05 test accuracy. This compares with 71.97\% $\pm$ 0.24 reported for a related attentional baseline on the leaderboard \citep{zhang2018gaan}, indicating that our baseline performance is relevant.

ER embeddings on \textit{ogbn-arXiv} need to be exceptionally high-dimensional to achieve accurate ER estimates ($\sim$11,000 dimensions), hence we were unable to use them here. However, incorporating \emph{ER scalar} features into our GAT model yielded a statistically-significant improvement of 72.14\% $\pm$ 0.03 test accuracy. \emph{Hitting time} features improve this result further to {\bf 72.25}\% $\pm$ 0.04 test accuracy. This demonstrates that our affinity-based metrics can yield useful improvements even on larger scale graphs, which are traditionally out of reach for methods like DGN \citep{beaini2021directional} due to computational complexity limitations. 

Reliable global leaderboarding with respect to \textit{ogbn-arXiv} is difficult, as state-of-the art approaches rely either on privileged information (such as raw text of the paper abstracts), incorporating node labels as features \citep{wang2021bag}, post-processing the predictions \citep{huang2020combining}, or various related tricks \citep{wang2021bag}. With that in mind, we report for convenience that the current state-of-the-art performance for \textit{ogbn-arXiv} without using raw text is 76.11\% $\pm$ 0.09 test accuracy, achieved by GIANT-XRT+DRGAT.

\subsubsection{Large scale graph regression: OGB-LSC PCQM4Mv1}

We finally include experimental results for one of the largest-scale publicly available graph regression tasks: the \textit{PCQM4Mv1} dataset from the OGB Large Scale Challenge \cite{hu2021ogb}. \textit{PCQM4M} is a quantum chemistry dataset spanning 4 million small molecules, with a task to predict the HOMO-LUMO gap, an important quantum-chemical property. It is anticipated that structural features such as ER could be of great help on this task, as the v1 version of it is provided without any structural information, and the molecule's geometry is assumed critical for predicting the gap. We report the single-model validation performance on this dataset, in line with previous works \cite{godwin2022simple,Ying2021DoTR,addanki2021large}.

\textit{PCQM4Mv1} comprises molecular graphs which consist of bonds and atom types, and no 3D or 2D coordinates. We reuse the experimental setup and architecture from \citep{godwin2022simple}, with only one difference: appending the effective resistance to the edge features. Additionally, we compare against an equivalent model which uses molecular conformations estimated by RDKit as an additional feature. This gives us a baseline which leverages an explicit estimate of the molecular geometry.

\begin{table}[ht]
    \centering
    \scalebox{0.95}{
    \begin{tabular}{ccccc}
    \toprule
         {\bf Model} & {\bf \#Layers} & {\bf Noisy Nodes} & {\bf Random Features} & {\bf Validation MAE}  \\
     \midrule
         MPNN \cite{godwin2022simple} & 16 & Yes & No & 0.1249 $\pm$ 0.0003 \\
         MPNN \cite{godwin2022simple} & 50 & No & No & 0.1236 $\pm$ 0.0001 \\
         Graphormer \cite{Ying2021DoTR} & - & - & - & 0.1234 \\
         MPNN \cite{godwin2022simple} & 50 & Yes & No & 0.1218 $\pm$ 0.0001 \\
         MPNN \cite{godwin2022simple} & 32 & Yes & No & 0.1222 $\pm$ 0.0002 \\
         MPNN \cite{godwin2022simple} & 32 & No & Yes & 0.1237 $\pm$ 0.0003 \\
         MPNN \cite{godwin2022simple} & 32 & Yes & Yes & 0.1216 $\pm$ 0.0003 \\
         MPNN + Conformers \cite{addanki2021large} & 32 & Yes & No & 0.1212 $\pm$ 0.0001 \\ \midrule
         MPNN + ER (ours) & 32 & No & No & 0.1214 $\pm$ 0.0002\\
         MPNN + ER (ours) & 32 & Yes & No & \textbf{0.1197 $\pm$ 0.0002}\\
     \bottomrule
    \end{tabular}
    }
    \caption{Single-model OGBG-PCQM4Mv1 Results}
    \label{tab:pcqm4m}
\end{table}

Our results are summarised in Table \ref{tab:pcqm4m}. We once again see a powerful synergy of effective resistance-endowed GNNs and Noisy Nodes \cite{godwin2022simple}, allowing us to significantly reduce the number of layers (to 32) and outperform the 50-layer MPNN result in \cite{godwin2022simple}. Further, we improve on the single-model performance of both the Graphormer \cite{Ying2021DoTR} (which won the original PCQM4M contest after ensembling), and an equivalent model to ours which uses molecular conformers from RDKit. This illustrates how ER features can be competitive in geometry-relevant tasks even against features that inherently encode an estimate of the molecule's spatial geometry.

Lastly, we remark that, to the best of our knowledge, our result is the \emph{best published single-model result} on the large-scale PCQM4M-v1 benchmark to date, and the only single model result with validation MAE under 0.120. We hope this will inspire future investigation on affinity-related GNNs for molecular tasks, especially in settings where spatial geometry is not reliably available.

%\subsection{Link detection}
%The next set of experiments concerns the problem of link %detection. For this, we use the community detection %datasets. (\newtodo{Include more information about the %exact datasets we use.})

\vspace{-10pt}
\section{Conclusions}
In this paper, we proposed a message passing network 
based on random walk based affinity measures. %We provide proofs of why these features lead to more expressive GNN architectures than most traditional GNN models. We further devised efficient algorithms to compute these quantities and evaluated their performance on a range of different node- and graph-predictive tasks, including the multitask PNA benchmark and 3 different tasks from the OGB suite including graphs of different scales. We show that the proposed graph topological features lead to improved performance in many of these tasks, in particular those pertaining to graph classification and regression. In particular, our MPNN model achieves the best known single-model validation MAE on OGB-LSC-PCQM4Mv1 at the time of writing.
We believe that the comprehensive theoretical and practical results presented in our paper have solidified affinity-based computations as a strong component of a graph representation learner's toolbox. Our proposal carefully balances theoretical expressive power, empirical performance, and scalability to large graphs---while most previous proposals struggle in at least one of the above---offering an attractive avenue for future work. Specifically, in future work we would like to see variants of GNN message functions that explicitly make use of affinity-based computations, rather than providing them as additional hints to the model.

\bibliographystyle{plain}  
\bibliography{bibliography}

\newpage
\appendix

\section{Hyperparameters for PNA dataset.}
In this section we provide the hyperparameters used for the different models on the PNA multitask benchmark. We train all models for 2000 steps and with 3 layers. The remaining hyperparameters for hidden size of each layer, learning rate, number of message passing steps (only valid for MPNN models), number of rotation matrices and same example frequency (when relevant) are provided in Table~\ref{tab:pna_hyperparams}.

\begin{table}[hb]
    \centering
    \begin{tabular}{cccccc}
    \toprule
         {\bf Model} & {\bf \#Hidden} & {\bf Learning} & {\bf \#MP} & {\bf \#Rotation} & {\bf \#Same} \\
          & {\bf size} & {\bf rate} & {\bf steps} & {\bf matrices} & {\bf examples} \\
     \midrule
         GAT & 64 & $10^{-4}$ & - & - & - \\
         GCN & 64 & $10^{-4}$ &  - & - & - \\
         DGN & 256 & $10^{-3}$ &  - & - & - \\
         MPNN & 256 & $10^{-3}$ & 2 & - & - \\
         ER GNN & 128 & $10^{-3}$ & 2 & - & - \\
         ER (node) embeddings & 64 & $10^{-3}$ & 1 & - & - \\
         ER (edge) embeddings & 256 & $10^{-3}$ & 2 & - & - \\
         ER (edge) embeddings & 256 & $10^{-3}$ & 2 & - & - \\
         All ER features & 256 & $10^{-4}$ & 2 & 23 & 9 \\
         HT + ER embeddings (rand rot) & 512 & $10^{-4}$ & 2 & 23 & 4 \\
     \bottomrule
    \end{tabular}
    \caption{Training hyperparameters for PNA dataset}
    \label{tab:pna_hyperparams}
\end{table}

\section{Omitted Proofs}
\label{apx:proofs}
\begin{lemma}[Restatement of \Cref{lem:embeddist}]
 For any pair of nodes $u,v$, we have $\|\resembed{u} - \resembed{v}\|_2^2 = \er(u,v)$.
\end{lemma}
\begin{proof}
\begin{align*}
    \|\resembed{u} - \resembed{v}\|_2^2 &= \|C^{1/2} B L_G^{-1} (\ind_u-\ind_v)\|_2^2 \\
    &= (\ind_u-\ind_v)^T L^{\dagger} (B^T C B) L^{\dagger} (\ind_u-\ind_v) \\
    &= (\ind_u-\ind_v)^T L^{\dagger} L L^{\dagger} (\ind_u-\ind_v) \\
    &= (\ind_u-\ind_v)^T L^{\dagger} (\ind_u-\ind_v) 
    = \er(u,v). \tag*{\qedhere}
\end{align*}
\end{proof}

\begin{corollary}[Restatement of \Cref{thm:jl-inner-product}]
For any fixed vectors
$\alpha,\beta \in \R^n$,
if we let $X := \sum_i \alpha_i {x}_i$, $\widehat{X} := \sum_i \alpha_i \widehat{x}_i$ and similarly 
$Y := \sum_i \beta_i {x}_i$, 
$\widehat{Y} := \sum_i \beta_i \widehat{x}_i$; then:
\[
\left|
\langle X, Y \rangle 
- \langle \widehat{X}, \widehat{Y} \rangle \right|
\le \frac{\epsilon}{2} \left(\|X\|^2 + \|Y\|^2\right).
\]
\end{corollary}
\begin{proof}
Since 
\(
\langle X, Y \rangle 
= 
\frac14 \left (\|X + Y\|^2 - \|X - Y\|^2\right)
\), we can upper-bound $\big|
\langle X, Y \rangle 
- \langle \widehat{X}, \widehat{Y} \rangle 
\big|$ as:
\begin{align*}
\frac14\left( |\|X + Y\|^2 -\|\widehat{X} + \widehat{Y}\|^2 |
+ |\|X - Y\|^2 -\|\widehat{X} - \widehat{Y}\|^2 |\right)
\\ 
\le \frac{\epsilon}{4} \left( \|X + Y\|^2 
+ \|X - Y\|^2  \right) 
= \frac{\epsilon}{2} \left(\|X\|^2 + \|Y\|^2\right).
\tag*{\qedhere}
\end{align*}
\end{proof}

\begin{lemma}[Restatement of \Cref{thm:hit-to-resembed}]
${H}_{u,v} = 2 M 
\langle \resembed{v}-\resembed{u}, \resembed{v} - {\mathbf{p}}\rangle$
where $\mathbf{p} := \sum_u \pi_u \resembed{u}$. 
\end{lemma}
\begin{proof}
Consider the following expression of hitting times
in terms of commute times by~\cite{tetali91}.
\begin{equation}
H_{u,v} 
= \frac12 \left[
K_{u,v} + \sum_i \pi_i \left( K_{v,i} - K_{u,i}  \right)
\right].
\label{eq:h-to-k}
\end{equation}
Dividing both sides of 
\cref{eq:h-to-k} and using the relation
$K_{u,v} = 2 M \er(u,v)$, we see that:
\begin{align}
    \frac{1}{2 M} H_{u,v} 
    = & \frac12 \left[ \er(u,v) + \sum_i \pi_i \left( \er(v, i)
    - \er(u, i) \right) \right] \notag \\
    = & \frac12 \left[ \|\resembed{u} - \resembed{v}\|^2
    + \sum_i \pi_i \left(\|\resembed{v} - \resembed{i}\|^2
    - \|\resembed{u} - \resembed{i}\|^2\right)\right].
    \label{eq:huv-part-1}
\end{align}
Let's focus on the inner summation. After expanding out the squared
norms, we see that:
\begin{align*}
    \sum_i & \pi_i \left(\|\resembed{v} - \resembed{i}\|^2
    - \|\resembed{u} - \resembed{i}\|^2\right) \\
    = & \sum_i \pi_i \left(\|\resembed{v}\|^2 - \|\resembed{u}\|^2\right)
    - 2\sum_i \langle \resembed{v} - \resembed{u}, \resembed{i}\rangle\\
    = & \sum_i \pi_i 
    \left[
    \left(\|\resembed{v}\|^2 - \|\resembed{u}\|^2\right)
    - 2\sum_i \langle \resembed{v} - \resembed{u}, \resembed{i}\rangle
    \right]\\
    = & \left(\|\resembed{v}\|^2 - \|\resembed{u}\|^2\right)
    - 2 \langle \resembed{v} - \resembed{u}, \sum_i \pi_i \resembed{i}\rangle \\
    = & \left(\|\resembed{v}\|^2 - \|\resembed{u}\|^2\right)
    - 2 \langle \resembed{v} - \resembed{u}, \mathbf{p} \rangle.
\end{align*}
Substituting this back into \cref{eq:huv-part-1}, 
we can express 
$\frac{1}{2 M} H_{u,v}$ as:
\begin{align*}
    \frac12 \left( 
    \|\resembed{v}-\resembed{u}\|^2
    + \|\resembed{v}\|^2 - \|\resembed{u}\|^2
    - 2 \langle \resembed{v}-\resembed{u}, \mathbf{p}\rangle
    \right) \\
    = \|\resembed{v}\|^2 - \langle \resembed{u}, \resembed{v}\rangle
    -
    \langle \resembed{v}-\resembed{u}, \mathbf{p}\rangle 
    = \langle \resembed{v} - \resembed{u}, \resembed{v} - \mathbf{p}
    \rangle. \tag*{\qedhere}
\end{align*}
\end{proof}

% \paragraph{Hyperparameters.}
% \textit{GAT}: hidden size 64; learning rate	$10^{-4}$; num layers 3; num training steps 2000	\\
% \textit{GCN}: hidden size 64; learning rate	$10^{-4}$; num layers 3; num training steps 2000	\\
% \textit{MPNN}: hidden size 256; learning rate 0.001; message passing steps 2; num layers 3; num training steps 2000	\\
% \textit{DGN}: hidden size 256; learning rate 0.001; num layers 3; num training steps 2000	\\
% \textit{ER GNN}: hidden size 128; learning rate 0.001; message passing steps 2; num layers 3; num training steps 2000	\\
% \textit{ER (node) embeddings}: hidden size 64; learning rate	0.001; message passing steps 1;	num layers 3; num training steps 2000 \\
% \textit{ER (edge) embeddings}: hidden size 256; learning rate 0.001; message passing steps 2; num layers 3; num training steps 2000	\\
% \textit{All ER features}: hidden size 256; learning rate 0.0001; message passing steps 2; num layers 3; num training steps 2000; num 
% rotation matrices 23; same example frequency 9\\

\section{Comparison: Effective Resistances vs. Shortest Path Distances}
Given that effective resistance (ER) captures times associated with random walks in a graph, it is tempting to ask how effective resistances compare to shortest path distances (SPDs) between nodes in a graph. Indeed, for some simple graphs, e.g., trees, shortest path distances and effective resistances turn out to be identical. However, in general, effective resistances and shortest path distances behave quite differently.

Nevertheless, it is tempting to ask how effective resistance features compare to SPD features in GNNs, especially as there have been a number of recent model architectures that make use of SPD features (e.g., Graphormer~\citep{YingCLZKHSL21}, Position-Aware GNNs~\citep{you2019position}, DE-GNN~\citep{LiWWL20}). 
We first note that the most natural direct comparison of our ER-based MPNNs with SPD-based networks does not quite make sense. The reason is that the analogous comparison would be to determine the effect of replace ERs with SPDs as features in our MPNNs. However, since our networks only use ER features \underline{along edges of the given graph}, the corresponding SPD features would then be trivial (as the SPD between two nodes directly connected by an edge in the graph is 1, resulting in a constant feature on every edge)!

As a result, graph learning architectures that use SPDs typically either (a.) use a densely-connected network (e.g., Graphormer~\citep{YingCLZKHSL21}, which uses a densely-connected attention mechanism) that incurs $O(n^2)$ overhead, or (b.) pick a small set of \emph{anchor nodes} or \emph{landmark nodes} to which SPDs from all other nodes are computed and incorporated as node features (e.g., Position-Aware GNNs~\citep{you2019position}, DE-GNN~\citep{LiWWL20}). We stress that the former approach generally modifies the graph (by connecting all pairs of nodes) and therefore does not fall within the standard MPNN approach, while the latter includes architectures that fall within the MPNN paradigm.

\subsection{Empirical Results}
In an effort to empirically compare the expressivity of ER features with that of SPD features, we once again perform experiments on the PNA dataset, picking the following baselines that make use of SPD features:
\begin{itemize}
    \item The first baseline is roughly an MPNN with \emph{Graphormer-based features}. More precisely, it is a densely-connected MPNN with SPDs \emph{from the original graph} as edge features. In order to retain the structure of the original graph, we also use additional edge features to indicate whether or not an edge in the dense (complete) graph is a true edge of the original graph. We also explore the use of the \emph{centrality encoding} (in-degree and out-degree embeddings) from Graphormer as additional node features.
    \item The second baseline is the Position-Aware GNN (P-GNN), which makes use of ``anchor sets'' of nodes and encodes distances to these nodes.
\end{itemize}
The results of these baselines are shown in \Cref{tab:pna}. In particular, we note that our ER-based MPNNs outperform all aforementioned baselines.
\begin{table}[ht]
    \centering
    \begin{tabular}{cc}
         \toprule
         {\bf Model} & {\bf Average score}\\
         \midrule
         \rowcolor{Yellow} *MPNN + CE & -2.728\\
         \rowcolor{Yellow} *MPNN (dense) + SPD & -2.157\\
         \rowcolor{Yellow} *MPNN (dense) + CE + SPD & -2.107 \\
         \rowcolor{Orange} *P-GNN & -2.650\\
         {\bf MPNN w/ ER (edge) embedding} & {\bf -2.789} \\
         {\bf MPNN w/ all ER features} & {\bf -3.106}\\
         \bottomrule
    \end{tabular}
    \caption{Results on the PNA dataset for MPNNs with Graphormer-based features (yellow) as well as SPD-based P-GNNs (orange). 
    %Note that ER embeddings still outperform MPNNs with Graphormer features as well as the SPD-based P-GNN model architecture.
    }
    \label{tab:pna}
\end{table}

\subsection{Theory: ER vs. SPD}
In addition to experimental results, we would like to provide some theory for why effective resistances can capture structure in GNNs that SPDs are unable to.

We will call an initialization function $u \mapsto \mathbf{h}_u^{(0)}$ on nodes of a graph \emph{node-based} if it assigns values that are independent of the edges of the graph. Such an initialization is, however, allowed to depend on node identities (e.g., for the single-source shortest path problem from a source $s$, one might find it natural to define $\mathbf{h}_s^{(0)} = 0$ and $\mathbf{h}_u^{(0)} = +\infty$ for all $u\neq s$).

%Note that any reasonable initialization function in practice would be a node-based %function, 

%does not depend on the connectivity features of the graph. In particular, the 
%initialization should be independent of the edges. This is a reasonable

%depends only on the node and its features, and nothing else. 
%In particular, the initial value for node $u$ should only be a function of $u$ and %$u$'s features. For SSSP problem, one might think of the initial value as $0$ if $u = % s$ and $+\infty$ otherwise.
%

Consider the task of computing ``single-source effective resistances,'' i.e., the effective resistance from a particular node to every other node.  We show that a GNN with a limited number of message passing steps cannot possibly learn single-source effective resistances, even to nearby nodes.
\begin{theorem}
 Suppose we fix $k > 0$. Then, given any node-based initialization function $\mathbf{h}_u^{(0)}$,  it is impossible for a GNN  
 to compute single-source effective resistances from a given node $w$ to any nodes within a $k$-hop neighborhood.

 More specifically, for any update rule
 \begin{align}
 \begin{split}
\mathbf{m}_{uv}^{(t+1)} &= \psi_{t+1}\left(\mathbf{h}_u^{(t)}, \mathbf{h}_v^{(t)}, f_e(\mathbf{x}_{uv})\right) \\
     \mathbf{h}_u^{(t+1)} &= \phi_{t+1}\left(\mathbf{h}_u^{(t)}, f\left(\left\{\mathbf{m}_{uv}: v\in\mathcal{N}(u)\right\}\right)\right)
 \end{split}, \label{eq:updaterule}
 \end{align}

 there exists a graph $G=(V,E)$ and $u\in V$ such that after $k$ rounds of message passing, $h_v^{(k)} \neq \er(u,v)$ for some $v\neq u$ within a $k$-hop neighborhood of $u$.
 
 On the other hand, there exists
 an initialization with respect to which $k$ rounds
 of message passing will compute the correct shortest
 path distances to all nodes within $k$-hop neighborhood.
\end{theorem}

Note that the assumption on the initialization function in the above theorem is reasonable because enabling the use of arbitrary, unrestricted functions would allow for the possibility of precomputing effective resistances in the graph and trivially incorporating them as node features, which would defeat the purpose of computing them using message-passing. 

We now prove the theorem.
\begin{proof}

Consider the following set of graphs, each on $4k+1$ nodes:

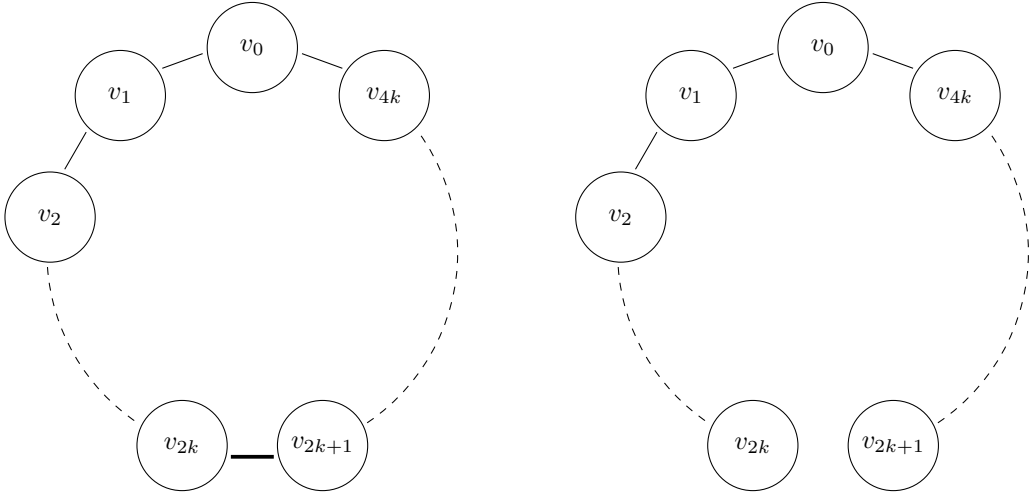
\begin{figure}[h]
\centering
\begin{tikzpicture}[scale=0.65]

\def \n {9}
\def \radius {4.2cm}
\def \margin {14} % margin in angles, depends on the radius

\foreach \s in {0,...,1}
{
  \node[draw, circle, minimum size=1.2cm] at ({90 + 360/\n * (\s)}:\radius) {$v_\s$};
  \draw[-, >=latex] ({90 + 360/\n * (\s)+\margin}:\radius)
    -- ({90 + 360/\n * (\s + 1)-\margin}:\radius);
}

\def \t {2}
\node[draw, circle, minimum size=1.2cm] at ({90 + 360/\n * (\t)}:\radius) {$v_\t$};
\draw[dashed, -, >=latex] ({90 + 360/\n * (\t)+\margin}:\radius)
arc ({90 + 360/\n * (\t)+\margin}:{90 + 360/\n * (\t+2)-\margin}:\radius);

\def \k {4}
\def \kp {5}
\node[draw, circle, minimum size=1.2cm] at ({90 + 360/\n * (\k)}:\radius) {$v_{2k}$};
\draw[-, >=latex, line width=0.5mm] ({90 + 360/\n * (\k)+\margin}:\radius)
-- ({90 + 360/\n * (\kp)-\margin}:\radius);

\node[draw, circle, minimum size=1.2cm] at ({90 + 360/\n * (\kp)}:\radius) {$v_{2k+1}$};
\draw[dashed, -, >=latex] ({90 + 360/\n * (\kp)+\margin}:\radius)
arc ({90 + 360/\n * (\kp)+\margin}:{90 + 360/\n * (\kp+3)-\margin}:\radius);

\node[draw, circle, minimum size=1.2cm] at ({90 + 360/\n * (\kp+3)}:\radius) {$v_{4k}$};
\draw[-, >=latex] ({90 + 360/\n * (\kp+3)+\margin}:\radius)
-- ({90 + 360/\n * (\kp+4)-\margin}:\radius);

\end{tikzpicture} \qquad\qquad
\begin{tikzpicture}[scale=0.65]

\def \n {9}
\def \radius {4.2cm}
\def \margin {14} % margin in angles, depends on the radius

\foreach \s in {0,...,1}
{
  \node[draw, circle, minimum size=1.2cm] at ({90 + 360/\n * (\s)}:\radius) {$v_\s$};
  \draw[-, >=latex] ({90 + 360/\n * (\s)+\margin}:\radius)
    -- ({90 + 360/\n * (\s + 1)-\margin}:\radius);
}

\def \t {2}
\node[draw, circle, minimum size=1.2cm] at ({90 + 360/\n * (\t)}:\radius) {$v_\t$};
\draw[dashed, -, >=latex] ({90 + 360/\n * (\t)+\margin}:\radius)
arc ({90 + 360/\n * (\t)+\margin}:{90 + 360/\n * (\t+2)-\margin}:\radius);

\def \k {4}
\def \kp {5}
\node[draw, circle, minimum size=1.2cm] at ({90 + 360/\n * (\k)}:\radius) {$v_{2k}$};

\node[draw, circle, minimum size=1.2cm] at ({90 + 360/\n * (\kp)}:\radius) {$v_{2k+1}$};
\draw[dashed, -, >=latex] ({90 + 360/\n * (\kp)+\margin}:\radius)
arc ({90 + 360/\n * (\kp)+\margin}:{90 + 360/\n * (\kp+3)-\margin}:\radius);

\node[draw, circle, minimum size=1.2cm] at ({90 + 360/\n * (\kp+3)}:\radius) {$v_{4k}$};
\draw[-, >=latex] ({90 + 360/\n * (\kp+3)+\margin}:\radius)
-- ({90 + 360/\n * (\kp+4)-\margin}:\radius);

\end{tikzpicture}
\caption{Both of the above graphs are on $4k+1$ vertices, labeled $v_0, v_1, \dots, v_{4k}$. The only difference is a single edge, i.e., the graph on the left has an edge between $v_{2k}$ and $v_{2k+1}$, while the one on the right does not have this edge.}
\end{figure}

Let $V = \{v_0, v_1, \dots, v_{4k}\}$. The first graph $G = (V,E)$ is a cycle, while the second graph $G'=(V,E')$ is a path, obtained by removing a single edge from the first graph (namely, the one between $v_k$ and $v_{k+1}$). Suppose the edge weights are all 1 in the above graphs.

Let $w = v_0$ be the source and let $\{\vec{h}_{v}^{(0)}: v \in V\}$ be a ``local'' node feature initialization. Note that for any GNN (i.e., update and aggregation rules in \eqref{eq:updaterule}, add the formal update rule somewhere), the computation tree after $k$ rounds of message passing is identical for nodes $v_0, v_1, \dots, v_k, v_{3k+1}, v_{3k+2}, \dots, v_{4k}$ (i.e., the nodes within the $k$-hop neighborhood of $v_0$) in both $G$ and $G'$. This is because the only difference between $G$ and $G'$ is the existence of the edge between $v_{2k}$ and $v_{2k+1}$, and this edge is beyond a $k$-hop neighborhood centered at any one of the aforementioned nodes. Therefore, we will necessarily have that $\vec{h}_{v_i}^{(k)}$ is identical in both $G$ and $G'$ for $i=1,\dots,k,3k+1, 3k+2, \dots, 4k$.

However, it is easy to calculate the effective resistances in both graphs. In $G$, we have $\er_G(v_0, v_i) = \frac{i(4k+1-i)}{4k+1}$, while in $G'$, we have $\er_{G'}(v_0, v_i) = \min\{i, 4k+1-i\}$. Therefore, $\er_G(v_0, v_i) \neq \er_{G'}(v_0, v_i)$ for all $i=1,2, \dots, k, 3k+1, 3k+2, \dots, 4k$.

It follows that for any $i=1,2,\dots, k, 3k+1, 3k+2, \dots, 4k$, the execution of $k$ message passing steps of a GNN cannot result in $\vec{h}_{v_i}^{(k)} = \er(v_0, v_i)$ for both $G$ and $G'$, which proves the first claim of the theorem.

For the second part (regarding single-source shortest paths), observe that single-source shortest path distances can, indeed, be realized via aggregation and update rules for a message passing network. In particular, for $k$ rounds of message passing, it is possible to learn shortest path distances of all nodes within a $k$-hop neighborhood. Specifically, for a source $w$, we can use the following setup: Take $\mathbf{h}_w = 0$ and $\mathbf{h}_u = \infty$ for all $u\neq w$. Moreover, for any edge $(u,v)$, let the edge feature $\mathbf{x}_{uv} \in \R$ simply be the weight of $(u,v)$ in the graph. Then, take the update rule \eqref{eq:updaterule} with $f_e, \psi_{t+1}$ as identity functions and
\begin{align*}
    f_e(\mathbf{x}_{uv}) &= \mathbf{x}_{uv}\\
    \psi_{t+1}\left(\mathbf{h}_u^{(t)}, \mathbf{h}_v^{(t)}, f_e(\mathbf{x}_{uv})\right) &= \mathbf{h}_u^{(t)} + \mathbf{x}_{uv}\\
    f(S) &= \min_S \{s\in S\}\\
    \phi_{t+1}(a,b) &= \min\{a,b\}.
\end{align*}
It is clear that the above update rule simply simulates the execution of an iteration of the Bellman-Ford algorithm. Therefore, $k$ message passing steps will simulate $k$ iterations of Bellman-Ford, resulting in correct shortest path distances from the source $w$ for every node within a $k$-hop neighborhood.
\end{proof}

\end{document}